\documentclass[twoside]{article}

\usepackage{aistats2020}
\usepackage{mathtools}
\usepackage{amssymb}
\usepackage{amsthm}
\usepackage{amsmath}
\DeclarePairedDelimiter{\ceil}{\lceil}{\rceil}

\DeclareMathOperator*{\argmin}{argmin}

\usepackage{algpseudocode}
\usepackage[ruled]{algorithm2e}
\SetKwRepeat{Do}{do}{while}

\algnewcommand{\comment}[1]{\Comment{{#1}}}

\theoremstyle{remark}

\newtheorem{theorem} {Theorem}
\newtheorem{lemma} {Lemma}
\newtheorem{definition} {Definition}

\def\x{{\mathbf{x}}}

\def\u{{\mathbf{u}}}
\def\v{{\mathbf{v}}}

\newcommand{\mT}{\mathcal{T}}

\newcommand{\R}{\mathcal{R}}

\newcommand{\mK}{\mathcal{K}}
\newcommand{\mS}{\mathcal{S}}
\newcommand{\ball}{\mathcal{B}}

\newcommand{\E}{\mathbb{E}}

\newcommand{\reals}{\mathbb{R}}

\begin{document}

%

%

\twocolumn[

\aistatstitle{Improved Regret Bounds for\\ Projection-free Bandit Convex Optimization}

\aistatsauthor{ Dan Garber \And Ben Kretzu}

\aistatsaddress{ Technion - Israel Institute of Technology \And  Technion - Israel Institute of Technology} ]

\begin{abstract}
We revisit the challenge of designing online algorithms for the bandit convex optimization problem (BCO) which are also scalable to high dimensional problems. Hence, we consider algorithms that are \textit{projection-free}, i.e., based on the conditional gradient method whose only access to the feasible decision set, is through a linear optimization oracle (as opposed to other methods which require potentially much more computationally-expensive subprocedures, such as computing Euclidean projections). We present the first such algorithm that attains $O(T^{3/4})$ expected regret using only $O(T)$ overall calls to the linear optimization oracle, in expectation, where $T$ is the number of prediction rounds. This improves over the $O(T^{4/5})$ expected regret bound recently obtained by  \cite{Karbasi19}, and actually matches the current best regret bound for projection-free online learning in the \textit{full information} setting.
\end{abstract}

\section{INTRODUCTION}
In this work we are interested in the design of efficient algorithms for online learning \cite{Cesa06, HazanBook, Bubeck12} which lie at the intersection of two families of algorithms, each by its own studied quite extensively in recent years with many new and exciting discoveries. The first, is the family of online learning algorithms for the \textit{bandit convex optimization} problem, and the second is the family of so-called \textit{projection-free} algorithms, which is a term casually used to refer to algorithms which are based on the conditional gradient method (aka Frank-Wolfe method), a well known first-order method for continuous optimization. These algorithms are called projection-free since, as opposed to popular first-order alternatives such as the projected / proximal / mirror gradient methods, which require in many cases to solve computationally-expensive optimization problems over the feasible domain (i.e., the projection step, which for instance in case of Euclidean projection, amounts to minimizing a quadratic function over the feasible set), the conditional gradient method only requires to minimize a linear function over the feasible set, which in many cases is much more efficient.

The bandit feedback model is well motivated by natural settings in which the online learner, upon making his prediction, only observes the loss associated with his prediction, and cannot infer the loss of different actions. The projection-free model is mostly motivated by large-scale settings which involve high-dimensional decision sets with non-trivial structure, for which computing Euclidean / mirror projections, which are required by standard algorithms (e.g., the celebrated online gradient descent algorithm \cite{Zinkevich03} and its adaptation to the BCO setting \cite{Flaxman05}), is computationally impractical (e.g., convex relaxations for sets of low-rank matrices or polytopes with special combinatorial structure, see \cite{Jaggi13} and \cite{Hazan12} for many examples). Thus, the combination of these two basic ingredients, both concern the possibility of applying online algorithms to large-scale real-world problems, is of interest.

A first attempt to combine these two ingredient was recently made in \cite{Karbasi19}, who combined the Online Frank-Wolfe method, suggested in \cite{Hazan12}, with the framework introduced in \cite{Flaxman05} for reducing BCO to the full-information setting (also known as \textit{online convex optimization} (OCO)), to obtain an algorithm that achieves expected regret of $O(T^{4/5})$ (treating all quantities except for number of prediction rounds $T$ as constants), using overall $T$ calls to the linear optimization oracle of the feasible set. Unfortunately, this regret bound is higher than both the expected regret achieved by the original method of \cite{Flaxman05} (though \cite{Flaxman05} uses Euclidean projections), which is $O(T^{3/4})$, and the regret obtained by the current state-of-the-art projection-free method (at least for arbitrary convex sets \footnote{for feasible sets with specific structure such as polytopes or smooth sets there are other algorithms that obtain optimal regret bounds in $T$ \cite{GH16, Levy19}.}) for the \textit{full-information} setting \cite{Hazan12}, which is also $O(T^{3/4})$.

It is thus natural to ask whether there is a price to pay, in terms of the worst-case expected regret bound, for combining these two settings, or alternatively, whether it is possible to obtain the best of both worlds, and get a projection-free algorithm for BCO that matches the state-of-the-art for the full-information setting. 

In this work we show that the latter is the case, i.e., we give a projection-free algorithm for BCO which attains $O(T^{3/4})$ expected regret bound, and uses overall only $O(T)$ calls to the linear optimization oracle, in expectation, thus matching the current state-of-the-art for projection-free algorithms even in the full-information setting. See also Table \ref{table:compare}.

In terms of techniques, as in \cite{Karbasi19}, our method is also based on combining the BCO framework of \cite{Flaxman05} and the online Frank-Wolfe method \cite{Hazan12}. The main novelty in our algorithm and analysis is based on the simple idea of partitioning the prediction rounds into non-overlapping equally-sized blocks. Surprisingly, by carefully analyzing the variance of the gradient estimator on each block, this simple trick allows us to strike a better and crucial tradeoff between the accuracy to which the subproblems of the Regularized-Follow-the-Leader method (the meta online learning algorithm on which our work, as well as \cite{Hazan12,Karbasi19}, is based) could be solved (via the conditional gradient method), and the overall regret of the algorithm. This results in meeting the current state-of-the-art bound for projection-free online convex optimization over general sets (even with full information of the loss functions), while maintaining linear (in $T$) linear optimization oracle complexity.

\begin{table*}[!htb]\renewcommand{\arraystretch}{1.4}
{\small
\begin{center}
\caption{Comparison of regret bounds and optimization oracle complexity. Only dependence on $T$ is stated. 
  }
  \begin{tabular}{| c | c | c | c | c |} \hline
    METHOD  & FEEDBACK &  PROJECTION-FREE? & ORACLE COMPLEXITY & $\E[\textrm{regret}]$ \\ \hline
\cite{Flaxman05} & Bandit & x  & $T$ projections & $T^{3/4}$\\ \hline
\cite{Hazan12} & Full & \checkmark & $T$ linear opt. steps  & $T^{3/4}$\\ \hline
\cite{Karbasi19} & Bandit & \checkmark & $T$ linear opt. steps   & $T^{4/5}$\\ \hline
This work (Thm. \ref{thm:main}) & Bandit & \checkmark  & $T$ linear opt. steps  & $T^{3/4}$\\ \hline
  \end{tabular}
%
  \label{table:compare}
\end{center}
}
\vskip -0.2in
\end{table*}\renewcommand{\arraystretch}{1}

\subsection{Additional Related Work}
As discussed, both the subject of designing projection-free methods for continuous optimization and bandit convex optimization have been studied extensively in recent years.

\textbf{Projection-free Methods:}
the conditional gradient method, which is the basic technique in most so-called projection-free methods, dates back to the classical works of Frank and Wolfe \cite{FrankWolfe}, and Polak \cite{Polyak}. The method has regained interest in recent years, especially in the context of large scale optimization and machine learning, see for instance \cite{Jaggi13, Jaggi10, Jaggi13a, GH15, Dudik12a, Dudik12b, ShalevShwartz11, Laue12}, just to name a few. There is also a recent effort to prove faster rates for simple variants of the method, usually under the assumption that the objective function is strongly convex (or a slightly weaker assumption) and assuming the feasible set admits certain structure (e.g., polytope, strongly convex set, bounded positive semidefinite cone, etc.), see for instance \cite{GH16, lacoste2015linear_fw, GH15, G16b, G16c,Allen17}.
\cite{Hazan12} were the first to suggest an algorithm for online convex optimization based on the conditional gradient method. Their method makes a single call to the linear optimization oracle on each round and achieves regret bound of $O(T^{3/4})$ for convex loss functions with bounded gradients (note this is worse than the optimal bound of $O(\sqrt{T})$, achievable for instance via the projection-based online gradient descent method \cite{Zinkevich03}). To date, this regret bound is the state-of-the-art for arbitrary compact and convex feasible sets. \cite{GH16} presented projection-free algorithms for OCO with optimal dependence on $T$ (i.e., $\sqrt{T}$), in case the feasible set is a polytope. Very recently, \cite{Levy19} suggested a regret-optimal algorithm for OCO in case the feasible set is smooth, however, as opposed to previous works, with an algorithm that is not based on the conditional gradient method.

\textbf{Bandit Convex Optimization:} following the work \cite{Flaxman05}, which presented an algorithm with $O(T^{3/4})$ expected regret bound for convex and Lipschitz loss functions, several other works obtained improved bounds, mostly under an additional smoothness assumption on the losses, see for instance \cite{Saha11, Dekel15, Mohri16, Hazan14}, In particular, in a recent effort, a series of works accumulated to a regret-optimal algorithm for BCO , achieving $\tilde{O}(\sqrt{T})$ regret \cite{Bubeck15bandit, Bubeck16bandit, Bubeck17bandit, Hazan16bandit}. Importantly, all these works which improve upon the $O(T^{3/4})$ bound of \cite{Flaxman05}, are based on much more complicated algorithms with running time either exponential in the dimension of the problem and $T$, or polynomial with a high-degree polynomial, and hence have impractical running times for large-scale problems. On the other-hand, in \cite{Flaxman05}, the only non-trivial operation is that of computing a Euclidean projection, which, as we show in this work, can be roughly speaking, replaced with a linear optimization step.

Finally, a special case of BCO in which all loss functions are linear was also studied extensively due to its special structure, see for instance \cite{Auer02, Abernethy08, Dani2008price, Hazan16volumetric}.

We also refer the interested reader to the following excellent introductory books on online learning and online convex optimization \cite{Cesa06, HazanBook, Bubeck12}.

\section{PRELIMINARIES}
\subsection{Bandit Convex Optimization And Assumptions}

We recall that in the bandit convex optimization problem, an online learner is required to iteratively draw actions from a fixed feasible set $\mK\in\reals^n$.\footnote{For convenience we assume the linear space of interest is $\reals^n$, however naturally, any finite-dimensional Euclidean space will work.}  After choosing his action $\x_t\in\mK$ on round $t\in[T]$ ($T$ is assumed to be known beforehand), he observes his loss given by $f_t(\x_t)$, where $f_t:\reals^n\rightarrow\reals$ is convex over $\reals^n$ and chosen by an adversary. Importantly, besides the value $f_t(\x_t)$, the learner does not gain any additional knowledge of $f_t(\cdot)$. In this work, we assume the adversary is oblivious, i.e., the loss functions $f_1,\dots,f_T$ are chosen beforehand and do not depend on the actions of the learner.

The goal of the learner is to minimize the expected regret which is given by
\begin{align}
        \mathbb{E}[\mathcal{R}_{T}] := \sum_{t=1}^{T} \mathbb{E} [f_t(\mathbf{x}_{t})] - \min\limits_{\mathbf{x} \in \mathcal{K}} \sum_{t=1}^{T} f_t(\mathbf{x}).
\end{align}

In this work, in addition to assuming the loss functions are convex, we also make the standard assumptions that they have subgradients upper-bounded by $G$ in $\ell_2$ norm over the feasible set $\mK$, for some $G >0$. That is, $\forall t\in[T]~\forall\x\in\mK~\forall\mathbf{g}\in\partial{}f_t(\x)$: $\Vert{\mathbf{g}}\Vert_2\leq G$. Also, as in \cite{Flaxman05} we make the standard assumption that the feasible set $\mK$ is full dimensional, contains the origin, and that there exists scalars $r,R>0$ such that $r\ball^n\subseteq\mK\subseteq{}R\ball^n$, where $\ball^n$ denotes the unit Euclidean ball centered at the origin in $\reals^n$.

\subsection{Additional  Notation And Definitions}

We denote by $\mS^n$ the unit sphere in $\mathbb{R}^n$, and we write $\mathbf{u} \sim S^n$ and $\mathbf{u} \sim \ball^n$ to denote a random vector $\u$ sampled uniformly from $\mS^n$ and $\ball^n$, respectively. We denote by $\Vert \mathbf{x} \Vert$ the $\ell_2$ norm of the vector $\mathbf{x}$. 

Finally, for a compact and convex set $\mK\subset\reals^n$, which satisfies the above assumptions (i.e., $r\ball^n\subseteq\mK\subseteq{}R\ball^n$), and a scalar $0 < \delta \leq r$, we define the set $\mK_{\delta} := (1-\delta/r)\mK = \{(1-\delta/r)\x~|~\x\in\mK\}$. In particular, it holds that $\mK_{\delta}\subseteq\mK$ and for all $\x\in\mK_{\delta}$, $\x+\delta\ball^n\subseteq\mK$ (see \cite{HazanBook}).

We now recall some standard definitions from continuous optimization. For all definitions we assume $\mK$ is a convex and compact subset of $\reals^n$.

\theoremstyle{definition}
\begin{definition}{}
    We say that $f: \mathbb{R}^n \xrightarrow{} \mathbb{R}$ is $G$-Lipschitz over $\mK$ if $\forall \mathbf{x}, \mathbf{y} \in \mathcal{K}$:
    \begin{align}
        |f(\mathbf{x}) - f(\mathbf{y})| \leq G  \Vert \mathbf{x} - \mathbf{y} \Vert . \nonumber
    \end{align}
    Here we recall, that if $f$ is convex over $\reals^n$ with subgradients upper-bounded by $G$ in $\ell_2$-norm over $\mK$ , then $f$ is $G$-Lipschitz over $\mK$.
\end{definition}
\begin{definition}{}
    We say that $f: \mathbb{R}^n \xrightarrow{} \mathbb{R}$ is $\beta$-smooth over $\mK$ if $\forall \mathbf{x}, \mathbf{y} \in \mathcal{K}$:
    \begin{align}
        f(\mathbf{y}) \leq  f(\mathbf{x}) + \nabla f(\mathbf{x})^{\top} (\mathbf{x} - \mathbf{y}) + \frac{\beta}{2}  \Vert \mathbf{x} - \mathbf{y} \Vert ^2. \nonumber
    \end{align}
\end{definition}
\begin{definition}{}
    We say that $f: \mathbb{R}^n \xrightarrow{} \mathbb{R}$  is $\alpha$-strongly convex over $\mathcal{K}$ if $\forall \mathbf{x}, \mathbf{y} \in \mathcal{K}$:
    \begin{align}
        f(\mathbf{y}) \geq  f(\mathbf{x}) + \nabla f(\mathbf{x})^{\top} (\mathbf{x} - \mathbf{y}) + \frac{\alpha}{2}  \Vert \mathbf{x} - \mathbf{y} \Vert ^2. \nonumber
    \end{align}
\end{definition}
Let $\mathbf{x}^*$ be the unique minimizer of $f$, an $\alpha$-strongly convex function over $\mathcal{K}$. From the above definition and the first order optimally condition it follows that $\forall \mathbf{x} \in \mathcal{K}$:
\begin{align}\label{eq:strong_convexity}
     \frac{\alpha}{2}  \Vert \mathbf{x} - \mathbf{x}^* \Vert ^2 \leq f(\mathbf{x}) - f(\mathbf{x}^*). 
\end{align}

\subsection{Basic Algorithmic Ingredients}
In this section we introduce some basic and standard algorithmic tools on which our algorithm is based. 

\subsubsection{Regularized Follow The Leader}\label{sec:RFTL}
One component of our algorithm is a variant of Regularized Follow the Leader (RFTL), which is a well known algorithm for online convex optimization \cite{HazanBook, Shalev12}. The prediction on time t is according to the following rule
\begin{align}
    \mathbf{x}_t = \argmin\limits_{\mathbf{x} \in \mathcal{K} }\bigg{\{} \sum_{i=1}^{t-1} f_{i}(\mathbf{x}) + \R(\mathbf{x}) \bigg{\}}, \nonumber
\end{align}
where $\R(\mathbf{x})$ is a strongly convex function.

\begin{lemma}[Lemma 2.3 in \cite{Shalev12}]
\label{lemma:shalev_rftl}
For all $t\in[T]$ let $\mathbf{x}_t^* = \argmin\limits_{\mathbf{x} \in \mathcal{K} }\big{\{} \sum_{i=1}^{t-1} f_{i}(\mathbf{x}) + \R(\mathbf{x}) \big{\}}$. Then, $\forall \mathbf{x} \in \mathcal{K}$ it holds that
\begin{align*}
    \sum_{t=1}^{T} \left( f_t(\mathbf{x}_t^*) - f_t(\mathbf{x}) \right) &\leq \R(\mathbf{x}) - \R(\mathbf{x}_1^*) \\
    &+ \sum_{t=1}^{T} \left( f_t(\mathbf{x}_t^*) - f_t(\mathbf{x}_{t+1}^*) \right). \nonumber
\end{align*}
\end{lemma}
\subsubsection{Smoothed Loss Functions}\label{sec:smooth}
Another standard component of our algorithm is the use of a smoothed version of each loss function. We define the $\delta$-smoothing of a loss function $f$ by
\begin{align}
    \hat{f}_{\delta}(\mathbf{x}) = \mathbb{E}_{\mathbf{u} \sim \ball^n} \left[ f(\mathbf{x} + \delta \mathbf{u}) \right]. \nonumber
\end{align}

We now cite some several useful lemmas regarding smoothed functions.

\begin{lemma}[Lemma 2.1 in \cite{HazanBook}] \label{lemma:hazan_smooth}
    Let $f: \mathbb{R}^n \xrightarrow{} \mathbb{R}$ be convex and $G$-Lipschitz over a convex and compact set $\mK\subset\reals^n$. Then $\hat{f}_{\delta}$ is convex and $G$-Lipschitz over $\mK_{\delta}$, and $\forall \mathbf{x} \in \mathcal{K}_{\delta}$ it holds that $|\hat{f}_{\delta} (\mathbf{x}) - f(\mathbf{x})| \leq \delta G$. 
\end{lemma}

\begin{lemma}[Lemma 6.5 in \cite{HazanBook}] \label{lemma:hazan_gradient}
	$\hat{f}_{\delta}(\x)$ is differentiable and
    \begin{align}
        \nabla \hat{f}_{\delta}(\mathbf{x}) = \mathbb{E}_{\mathbf{u} \sim \mS^n} \left[ \frac{n}{\delta} f(\mathbf{x} + \delta \mathbf{u})\mathbf{u} \right]. \nonumber
    \end{align}
\end{lemma}

\begin{lemma} [see \cite{Bertsekas73}]
\label{lemma:bertsekas_grdient}
    Let $f: \mathbb{R}^n \xrightarrow{} \mathbb{R}$ be convex and suppose that all subgradients of $f$ are upper-bounded by $G$ in $\ell_2$-norm over a convex and compact set $\mK\subset\reals^n$. Then, for any $\x\in\mK_{\delta}$ it holds that $\Vert{\nabla{}f_{\delta}(\x)}\Vert \leq G$.
\end{lemma}

\section{ALGORITHM AND ANALYSIS}
As in \cite{Karbasi19}, our algorithm (see Algorithm \ref{alg:VRBCO} below) is based on combining the BCO framework of \cite{Flaxman05} with the Online Frank-Wolfe method of \cite{Hazan12}. That is, the algorithm applies the Regularized Follow the Leader meta-algorithm with Euclidean regularization (see Section \ref{sec:RFTL}), and uses the bandit feedback to construct unbiased estimates for the gradients of the smoothed losses, by sampling points in a sphere around the current iterate (see Section \ref{sec:smooth}). In order to avoid solving the RFTL optimization problem (which with the standard linearization trick of the smoothed losses and using Euclidean regularization, amounts to minimizing a quadratic function over the set $\mK_{\delta}$), we invoke the conditional gradient method (see Algorithm \ref{alg:Cg}), to solve this problem only to sufficient approximation using only linear optimization steps over the feasible domain $\mK_{\delta}$.

Very importantly, different from \cite{Karbasi19}, we partition the $T$ prediction rounds into non-overlapping blocks of size $K$ ($K$ is a parameter determined in the analysis), where on each block the iterate of the algorithm remains unchanged (though we use fresh samples for exploration on each round within a block). Essentially without loosing generality we assume that $T/K$ is an integer. This partition into blocks is important since as we show in the analysis, it allows us to solve the RFTL objective via the conditional gradient method to better accuracy, without incurring any substantial price in the regret or the overall linear oracle complexity.

It is also important to note that our algorithm is structured in a way that on each block $m$ in the run of the algorithm, the point $\x_{m-1}$ used for prediction, only takes into account the loss function revealed up to (and including) block $m-2$ (note $\x_{m-1}$ is an approximate minimizer of $\hat{F}_{m-1}(\x)$, which in turn depends only on the estimates $\hat{\mathbf{g}}_1,\dots,\hat{\mathbf{g}}_{m-2}$). Thus, in principle, Algorithm \ref{alg:VRBCO} does not have to wait after each block $m$ until the new iterate $\x_m$ is computed for the following block $m+1$ via Algorithm \ref{alg:Cg}. While Algorithm \ref{alg:VRBCO} uses $\x_{m-1}$ for prediction on block $m$, it can run Algorithm \ref{alg:Cg} in parallel, to simultaneously compute the next iterate $\x_m$ (which is independent of the gradient estimates obtained in block $m$). 

While this self-induced delay in information usage is not important for the theoretical complexity analysis, we believe it is of practical importance, since otherwise without this delay, Algorithm \ref{alg:VRBCO} would have to stop after each block and wait for Algorithm \ref{alg:Cg} to finish its computation, which can be potentially prohibitive in high-frequency prediction settings.

Finally, note that while the conditional gradient method is run over the shrunk set $\mK_{\delta}$, solving the linear optimization problem over $\mK_{\delta}$ is identical, up to scaling, to solving it over the original set $\mK$.

\begin{algorithm}[h]
  \KwData{horizon $T$, feasible set $\mathcal{K}$ with parameters $r,R$, block size $K$, step size $\eta$, smoothing parameter $\delta\in(0,r]$, tolerance parameter $\epsilon$}
  \KwResult{$\mathbf{y}_1, \mathbf{y}_2, \ldots ,\mathbf{y}_T$ }
  $\mathbf{x}_0 \gets $ arbitrary point in $\mK_{\delta}$, $\x_1 \gets \x_0$\\
  \For{$~ m = 1,\ldots,\frac{T}{K} ~$}{
    define  $\hat{F}_m(\mathbf{x}):= \eta \sum_{i=1}^{m-1} \mathbf{x}^{\top} \hat{\mathbf{g}}_i +   \Vert \mathbf{x} - \mathbf{x}_1 \Vert ^2 $\\
    \If{$m>1$}{ 
    run Algorithm \ref{alg:Cg} with set $\mathcal{K}_{\delta}$,  tolerance $\epsilon$, initial vector $\mathbf{x}_{m-1}$, and function $\hat{F}_{m}(\mathbf{x})$. Execute \textbf{in parallel} to following \textbf{for} loop over $s$
    }
    \For{$~ s = 1 ,\ldots, K ~$}{
    $\mathbf{u}_t$ $\sim S^n$ \comment{$t=(m-1)K+s$}\\
    play $\mathbf{y}_t \xleftarrow{} \mathbf{x}_{m-1} + \delta \mathbf{u}_t$ and observe $f_t(\mathbf{y}_t)$\\
    $\mathbf{g}_t$ $\xleftarrow{} \frac{n}{\delta} f_t(\mathbf{y}_t) \mathbf{u}_t$
    }
    $\hat{\mathbf{g}}_m$ $\xleftarrow{}$ $\sum_{s=1}^{K} \mathbf{g}_{(m-1)K+s}$\\
    \If{$m>1$}{
   $\mathbf{x}_{m} \gets$ output of Algorithm \ref{alg:Cg}
    }
  }
  \caption{Block Bandit Conditional Gradient Method}\label{alg:VRBCO}
\end{algorithm}

\begin{algorithm}[!]
  \KwData{feasible set $\mathcal{K}_{\delta}$, error tolerance $\epsilon$, initial vector $\mathbf{x}_{in}$, objective function $\hat{F}_m(\mathbf{x})$}
  \KwResult{$\mathbf{x}_{out}$ }
  $\mathbf{z}_1 \gets \mathbf{x}_{in}$, $\tau \gets 0 $\\
  \Do{$\nabla \hat{F}_m(\mathbf{z}_\tau)^{\top} (\mathbf{z}_{\tau} - \mathbf{v}_{\tau}) > \epsilon$}{
    $\tau \gets \tau + 1 $\\
    $ \mathbf{v}_\tau \in \argmin\limits_{\mathbf{x} \in \mathcal{K}_{\delta}} \{ \nabla \hat{F}_m(\mathbf{z}_\tau)^{\top} \cdot \mathbf{x} \} $\\
	$ \sigma_{\tau} = \argmin\limits_{\sigma \in [0, 1]}  \{ \hat{F}_m(\mathbf{z}_\tau + \sigma (\mathbf{v}_\tau - \mathbf{z}_\tau)) \} $ \comment{Line-search}\\
	$ \mathbf{z}_{\tau+1} = \mathbf{z}_\tau + \sigma_{\tau} (\mathbf{v}_\tau - \mathbf{z}_\tau) $ \comment{$\mathbf{z}_{\tau+1} \in \mathcal{K}_{\delta} $}
  }
  $\mathbf{x}_{out} \gets \mathbf{z}_{\tau}$
  \caption{Conditional Gradient with Stopping Condition}\label{alg:Cg}
\end{algorithm}

In the following, for any iteration (or block) $m$ of the outer-loop in Algorithm \ref{alg:VRBCO}, we denote by $L_m$ the overall number of iterations performed by the do-while loop of Algorithm \ref{alg:Cg}, when invoked on iteration $m$. In particular, note that $\sum_{i=1}^{T/K}L_i$ is the overall number of calls to the linear optimization oracle of $\mK$ throughout the run of Algorithm \ref{alg:VRBCO}.

\begin{theorem} [Main theorem] \label{thm:main}
    For all $ c > 0 $ such that $\frac{cT^{-1/4}}{r} \leq 1$, setting $\eta = \frac{2 c R}{nM} T^{-\frac{3}{4}}$, $\delta = c T^{-\frac{1}{4}} $, $\epsilon =  16R^2  T^{-\frac{1}{2}}$, $K = T^{\frac{1}{2}}$ in Algorithm \ref{alg:VRBCO}, guarantees that the expected regret is upper-bounded by
    {\small\begin{align}
        \mathbb{E}[\mathcal{R}_{T}] \leq & \left( 3  c  G + \frac{c R G }{r} + 6 G R + 4 \frac{c G^2 R}{nM} + 4\frac{RnM}{c } \right) T^{\frac{3}{4}}, \nonumber
    \end{align}}
    and that the expected overall number of calls to the linear optimization oracle is upper-bounded by
    \begin{align}
        \mathbb{E} \left[ \sum_{m=1}^{\frac{T}{K}} L_m \right] \leq & ~ \left( \frac{3}{4} + \frac{G c}{2 n M} + \frac{G^2 c^2 }{4 n^2 M^2} \right) T . \nonumber
    \end{align}
    In particular, if $\left( \frac{nM}{Gr} \right)^2 \leq T$ then, setting $c=\sqrt{\frac{nMr}{G}}$, we have
    {\small
    \begin{align}
        \mathbb{E}[\mathcal{R}_{T}] \leq & \left( 8 \frac{R}{r}  \sqrt{nMrG} + 6 G R + 4 G R \sqrt{\frac{rG}{nM}} \right) T^{\frac{3}{4}}, \nonumber
    \end{align}}
    and
    \begin{align}
        \mathbb{E} \left[ \sum_{m=1}^{\frac{T}{K}} L_m \right] \leq & ~ \left( \frac{3}{4} + \frac{1}{2}\sqrt{\frac{Gr}{nM}} + \frac{1}{4}\sqrt{\frac{Gr }{nM}} \right) T . \nonumber
    \end{align}
    
\end{theorem}

\subsection{Analysis}

For the purpose of the analysis, we define the auxiliary sequence $\{\x^*_m\}_{m=1}^{T/K}$ as $\x_m^* = \arg\min_{\x\in\mK_{\delta}}\hat{F}_m(\x)$, where $\hat{F}_m(\cdot)$ is as defined in Algorithm \ref{alg:VRBCO}. Note that this sequence corresponds to running the RFTL algorithm in blocks of length $K$, with respect to the feasible set $\mK_{\delta}$ (see Section \ref{sec:RFTL}).

The following lemma, which is crucial to obtain our improved regret bound, shows that the squared norm of the gradient estimator over a block of size $K$, as a first approximation, grows only linearly with the block size $K$.
\begin{lemma} \label{lemma:expectation_gradient}
    For any iteration (block) $m$ of the outer-loop in Algorithm \ref{alg:VRBCO} it holds that
    \begin{align}
        \mathbb{E} \left[  \Vert \hat{\mathbf{g}}_m \Vert  \right]^2 \leq \mathbb{E} \left[  \Vert \hat{\mathbf{g}}_m \Vert ^2 \right] \leq K \left(\frac{n M}{\delta}\right)^2 + K^2 G^2. \nonumber
    \end{align}
\end{lemma}
\begin{proof}
Fix some block $m$. For convenience, we denote $\mT_m = \{(m-1)K+1,\cdots,mK\}$ (i.e., the set of all rounds included in block $m$). 
It holds that
    \begin{align}
        \mathbb{E} \left[  \Vert \hat{\mathbf{g}}_m \Vert ^2 \right] & =   \mathbb{E} \left[ \Vert \sum_{t\in\mT_m} \mathbf{g}_t \Vert ^2 \right] \nonumber \\
        = & \mathbb{E} \left[\sum_{t\in\mT_m}  \Vert \mathbf{g}_t \Vert ^2 + \sum_{(i,j)\in\mT_m^2,i\neq j}\mathbf{g}_i^{\top} \mathbf{g}_j \right] \nonumber \\
        = & \mathbb{E} \left[\sum_{t\in\mT_m}  \Vert  \mathbf{g}_t \Vert ^2 \right] + \sum_{(i,j)\in\mT_m^2,i\neq j} \mathbb{E} \left[ \mathbf{g}_i^{\top} \mathbf{g}_j \right] . \nonumber 
    \end{align}	
    Since, conditioned on the iterate $\x_{m-1}$, $\forall i \neq j$  $\mathbf{g}_i$, $\mathbf{g}_j$ are independent random vectors, we have
    {\begin{align*}
        &\mathbb{E} \left[  \Vert \hat{\mathbf{g}}_m \Vert ^2 \right] = \mathbb{E} \left[\sum_{t\in\mT_m} \Vert  \mathbf{g}_t \Vert ^2 \right] \\
        &+\sum_{(i,j)\in\mT_m^2,i\neq j}\E\left[{ \mathbb{E} [ \mathbf{g}_i^{\top}|\x_{m-1}] \mathbb{E}[\mathbf{g}_j |\x_{m-1}]}\right] .  \nonumber
    \end{align*}}
    Using Lemma \ref{lemma:bertsekas_grdient} we have that for all $t\in\mT_m$, $\Vert{\mathbb{E} [\mathbf{g}_t|\x_{m-1}]}\Vert = \Vert{\nabla {\hat{f}}_{t,\delta} (\mathbf{x}_{m-1})}\Vert \leq G$.
    Since $\max_{\mathbf{x} \in \mathcal{K}}  \Vert f(\mathbf{x}) \Vert  \leq M$, we also have $  \Vert \mathbf{g}_t \Vert  \leq \frac{n}{\delta}  \Vert f_t(\mathbf{y}_t) \Vert   \Vert \mathbf{u}_t \Vert  \leq \frac{nM}{\delta}$, and thus,
    {\small\begin{align*}
        &\mathbb{E} \left[\sum_{t\in\mT_m} \Vert  \mathbf{g}_t \Vert ^2 \right] +\sum_{(i,j)\in\mT_m^2,i\neq j}\E\left[{ \mathbb{E} [ \mathbf{g}_i^{\top}|\x_{m-1}] \mathbb{E}[\mathbf{g}_j |\x_{m-1}]}\right] \\ 
        &\leq  K \left(\frac{n M}{\delta}\right)^2 + \left(K^2 - K \right) G^2 \leq K \left(\frac{n M}{\delta}\right)^2 + K^2G^2. \nonumber
    \end{align*}}
      Finally, the inequality $\mathbb{E} \left[  \Vert \hat{\mathbf{g}}_m \Vert  \right]^2 \leq \mathbb{E} \left[  \Vert \hat{\mathbf{g}}_m \Vert ^2 \right] $ stated in the lemma follows from using Jensen's inequality.
\end{proof}


The following lemma combines the RFTL regret bound with the unbiased gradient estimates of the smoothed loss functions, and upper-bounds the expected regret of Algorithm \ref{alg:VRBCO}.

\begin{lemma}\label{lemma:RFTL}
Suppose that throughout the run of Algorithm \ref{alg:VRBCO}, for all blocks $m = 1, \dots, \frac{T}{K}$ it holds that $\hat{F}_{m}(\mathbf{x}_{m}) -  \hat{F}_{m}(\mathbf{x}_{m}^*) \leq \epsilon$. Then,  
the expected regret of the algorithm is upper-bounded by
    \begin{align}
        \mathbb{E}[\mathcal{R}_{T}] \leq & \left( 3 \delta G + \delta RG/r + G  \sqrt{\epsilon} +   \frac{\eta G \sqrt{K} n M}{\delta} \right.  \nonumber \\
        & \left. ~ +  \eta \left( \frac{n M}{\delta} \right)^2  + 2 \eta K G^2 \right)T + \frac{4R^2}{\eta} . \nonumber
    \end{align}
\end{lemma}
Define $\mathbf{x}^* \in \argmin\limits_{\mathbf{x} \in  \mathcal{K}} \sum_{t=1}^{T} f_t(\mathbf{x})$, $\tilde{\mathbf{x}}^* =  (1- \delta/r)\mathbf{x}^*$, $\forall t : m(t) := \ceil*{\frac{t}{K}}$. Recall that throughout any block $m$,  Algorithm \ref{alg:VRBCO} predicts according to  $\mathbf{x}_{m-1}$.
\begin{proof}
    It holds that 
    {\small\begin{align}
        \mathbb{E}[&\mathcal{R}_{T}]  =  \sum_{t=1}^{T} \mathbb{E} [f_t(\mathbf{y}_{t})] - \sum_{t=1}^{T} f_t(\mathbf{x}^*) \nonumber \\
        = & \sum_{t=1}^{T} \mathbb{E} [f_t(\mathbf{y}_{t})] - \sum_{t=1}^{T} \mathbb{E} [f_t(\mathbf{x}_{m(t)-1})] + \sum_{t=1}^{T} \mathbb{E} [f_t(\mathbf{x}_{m(t)-1})] \nonumber \\
        & - \sum_{t=1}^{T} f_t(\tilde{\mathbf{x}}^*) + \sum_{t=1}^{T} f_t(\tilde{\mathbf{x}}^*) - \sum_{t=1}^{T} f_t(\mathbf{x}^*). \label{eq:full_regret_rftl_lemma}
    \end{align}}
    $f_t$ is $G$-Lipschitz, and thus we have that
    {\small\begin{align}
        \sum_{t=1}^{T} \mathbb{E} [f_t(\mathbf{y}_{t})]  & - \sum_{t=1}^{T} \mathbb{E} [f_t(\mathbf{x}_{m(t)-1})] \nonumber \\
        & = \sum_{t=1}^{T} \mathbb{E} [f_t(\mathbf{x}_{m(t)-1} + \delta \mathbf{u}_{t}) - f_t(\mathbf{x}_{m(t)-1})] \nonumber \\
        & \leq \sum_{t=1}^{T} \mathbb{E} [G ~  \Vert \delta \mathbf{u}_{t} \Vert ] \leq \delta G T. \label{eq:bandit_feedback_regret_rftl_lemma}
    \end{align}}
    Also,
    \begin{align}
        \sum_{t=1}^{T} f_t(\tilde{\mathbf{x}}^*) & - \sum_{t=1}^{T} f_t(\mathbf{x}^*) \leq \sum_{t=1}^{T} G  \Vert  \tilde{\mathbf{x}}^* - \mathbf{x}^*  \Vert  \nonumber \\
        & = \sum_{t=1}^{T} G \Vert (1-\delta/r) \mathbf{x}^* - \mathbf{x}^* \Vert \leq \delta R G T/r. \label{eq:squeeze_regret_rftl_lemma}
    \end{align}    
    Now, we need to obtain an upper bound on $\sum_{t=1}^{T} \mathbb{E} [f_t(\mathbf{x}_{m(t)-1})]  -  \sum_{t=1}^{T} f_t(\tilde{\mathbf{x}}^*)$. We will first take a few preliminary steps. Define for all $m \in \left[\frac{T}{K}\right]$ $\mathcal{F}_m = \{ \mathbf{x}_1, \hat{\mathbf{g}}_1, \dots, \mathbf{x}_{m-1}, \hat{\mathbf{g}}_{m-1} \} $- the history of all predictions and gradient estimates. Throughout the sequel we introduce the short notation $ \hat{\nabla}_{t,\delta,m(t)-1} = \nabla {\hat{f}}_{t,\delta}(\mathbf{x}_{m(t)-1})$.
    Since $\mathbf{g}_t$ is an unbiased estimator of $\nabla {\hat{f}}_{t,\delta}(\mathbf{x}_{m(t)-1}) = \hat{\nabla}_{t,\delta,m(t)-1}$, then $\mathbb{E} \left[ \mathbf{g}_t | \mathcal{F}_m \right] = \hat{\nabla}_{t,\delta,m(t)-1}$. Since  $\mathbf{x}_m^* = \argmin\limits_{\mathbf{x} \in (1-\delta/r) \mathcal{K}} \Big{\{} \hat{F}_{m}(\mathbf{x}) := \eta \sum_{i=1}^{m-1} \mathbf{x}^{\top} \hat{\mathbf{g}}_i +   \Vert \mathbf{x} - \mathbf{x}_1 \Vert ^2 \Big{\}} \nonumber$, we have that $\mathbb{E} \left[ \mathbf{x}_m^*| \mathcal{F}_m \right] = \mathbf{x}_m^*$. From both observations $\forall \mathbf{x} \in (1-\delta/r) \mathcal{K}$ and $\forall m \in \left[\frac{T}{K}\right] $, it holds that
    {\small\begin{align}
        \mathbb{E}  \left[ \hat{\mathbf{g}}_m^{\top}  (\mathbf{x}_{m}^* - \mathbf{x}) \right] & = \mathbb{E} \left[ \mathbb{E} \left[ \hat{\mathbf{g}}_m| \mathcal{F}_{m} \right] ^{\top}  (\mathbf{x}_{m}^* - \mathbf{x}) \right] \nonumber \\
        = & \mathbb{E} \left[ \sum_{t=(m-1)K+1}^{mK} \mathbb{E} \left[  \mathbf{g}_t| \mathcal{F}_{m(t)} \right] ^{\top}  (\mathbf{x}_{m(t)}^* - \mathbf{x}) \right] \nonumber \\
        = & \sum_{t=(m-1)K+1}^{mK}  \mathbb{E} \left[ \hat{\nabla}_{t,\delta,m(t)-1}^{\top} (\mathbf{x}_{m(t)}^* - \mathbf{x}) \right]. \label{eq:expectation_derivative_rftl_lemma}
    \end{align}}
    Using Lemma \ref{lemma:hazan_smooth} with the regularizer $\R(X) = \frac{ \Vert \mathbf{x} - \mathbf{x}_1  \Vert ^2}{\eta}$ and $\{\hat{\mathbf{g}}_m^{\top}\mathbf{x}\}_{m=1}^{T/k}$ as the (linear) loss functions, we have that $\forall \mathbf{x} \in (1-\delta/r) \mathcal{K}$,
    {\small\begin{align}
        \sum_{m=1}^{\frac{T}{K}} \hat{\mathbf{g}}_m^{\top} (\mathbf{x}_{m}^* - \mathbf{x}) \leq \sum_{m=1}^{\frac{T}{K}} \hat{\mathbf{g}}_m^{\top}  (\mathbf{x}_{m}^* - \mathbf{x}_{m+1}^*) +  \frac{1}{\eta}  \Vert \mathbf{x} - \mathbf{x}_1 \Vert ^2. \nonumber
    \end{align}}
    Since for all $m$, $\hat{F}_m(\mathbf{x})$ is $2$-strongly convex and $\hat{F}_{m}(\mathbf{x}_{m}^*) \leq \hat{F}_{m}(\mathbf{x}_{m+1}^*)$, using Eq. \eqref{eq:strong_convexity} we have that
    \begin{align}
       \Vert \mathbf{x}_{m}^*-&\mathbf{x}_{m+1}^* \Vert ^2  \leq \hat{F}_{m+1}(\mathbf{x}_{m}^*) - \hat{F}_{m+1}(\mathbf{x}_{m+1}^*) \nonumber \\
      & =  \hat{F}_{m}(\mathbf{x}_{m}^*) - \hat{F}_{m}(\mathbf{x}_{m+1}^*) + \eta \hat{\mathbf{g}}_{m}^{\top} (\mathbf{x}_{m}^* - \mathbf{x}_{m+1}^*)  \nonumber \\ 
      & \leq  \eta  \Vert \hat{\mathbf{g}}_{m} \Vert  \Vert (\mathbf{x}_{m}^* - \mathbf{x}_{m+1}^*) \Vert . \label{eq:absolut_dist}
    \end{align}
    From the above inequality we obtain $ \Vert \mathbf{x}_{m}^*-\mathbf{x}_{m+1}^* \Vert  \leq \eta  \Vert \hat{\mathbf{g}}_{m} \Vert $.
    From these three observations and Cauchy-Schwarz inequality, we have
    {\small\begin{align}
        \sum_{t=1}^{T} \mathbb{E} & \left[ \hat{\nabla}_{t,\delta,m(t)-1}^{\top}   (\mathbf{x}_{m(t)}^* - \tilde{\mathbf{x}}^*) \right] = \E\left[ \sum_{m=1}^{\frac{T}{K}} \hat{\mathbf{g}}_m^{\top}  (\mathbf{x}_{m}^* - \tilde{\x}^*) \right] \nonumber\\
        \leq & ~ \mathbb{E} \left[ \sum_{m=1}^{\frac{T}{K}} \hat{\mathbf{g}}_m^{\top}  (\mathbf{x}_{m}^* - \mathbf{x}_{m+1}^*) \right] + ~ \frac{1}{\eta}  \Vert \tilde{\mathbf{x}}^* - \mathbf{x}_1 \Vert ^2 \nonumber\\
        \leq & ~ \eta  \sum_{m=1}^{\frac{T}{K}} \mathbb{E} \left[ 
         \Vert \hat{\mathbf{g}}_m \Vert ^2 \right] + \frac{4R^2}{\eta} \nonumber\\
        \underset{(a)}{\leq} & ~  \eta T \left(\frac{n M}{\delta}\right)^2 + \eta K  T G^2 + \frac{4R^2}{\eta}. \label{eq:rftl_regret_rftl_lemma}
    \end{align}}
    Inequality (a) is due to Lemma \ref{lemma:expectation_gradient}. Using Lemma \ref{lemma:bertsekas_grdient} we have that for all $t\in[T]$, $\Vert{\nabla {\hat{f}}_{t,\delta} (\mathbf{x}_t)}\Vert \leq G$. Also, using Eq. \eqref{eq:strong_convexity} w.r.t. $\hat{F}_{m-1}$ and our assumption, $\hat{F}_{m-1}(\mathbf{x}_{m-1}) -  \hat{F}_{m-1}(\mathbf{x}_{m-1}^*) \leq \epsilon$, we have that 
    \begin{align}
         \sum_{t=1}^{T} \mathbb{E} &\left[ \hat{\nabla}_{t,\delta,m(t)-1}^{\top} (\mathbf{x}_{m(t)-1} - \mathbf{x}_{m(t)-1}^*) \right] \nonumber \\
         & \leq  G \sum_{t=1}^{T} \mathbb{E} \left[  \Vert (\mathbf{x}_{m(t)-1} - \mathbf{x}_{m(t)-1}^*) \Vert  \right] \nonumber \\
         & \leq G K \sum_{m=1}^{\frac{T}{K}}  \mathbb{E} \left[ \sqrt{\hat{F}_{m-1}(\mathbf{x}_{m-1}) - \hat{F}_{m-1}(\mathbf{x}_{m-1}^*)} \right]  \nonumber \\
         & \leq G T \sqrt{\epsilon}. \label{eq:cg_regret_rftl_lemma}
    \end{align}
    Using Eq. \eqref{eq:absolut_dist}, we have
    \begin{align}
         &\sum_{t=1}^{T} \mathbb{E} \left[ \hat{\nabla}_{t,\delta,m(t)-1}^{\top} (\mathbf{x}_{m(t)-1}^* - \mathbf{x}_{m(t)}^*) \right] \nonumber \\
         & \leq  G \sum_{t=1}^{T} \mathbb{E} \left[  \Vert (\mathbf{x}_{m(t)-1}^* - \mathbf{x}_{m(t)}^*) \Vert  \right]  \nonumber \\
         & \leq G K \eta \sum_{m=1}^{\frac{T}{K}}  \mathbb{E} \left[   \Vert \hat{\mathbf{g}}_{m-1} \Vert  \right]  \nonumber \\
         &\underset{(a)}{\leq}  G T \eta \left( \frac{\sqrt{K} n M}{\delta} + KG \right)
         . \label{eq:osb_regret_rftl_lemma}
    \end{align}
    Inequality (a) is due to Lemma \ref{lemma:expectation_gradient} and the fact that for all $a,b \in \reals^+$ it holds that $\sqrt{a+b} \leq \sqrt{a} + \sqrt{b}$. One last step before we will achieve the target bound, we require an upper bound on the regret w.r.t. the $\delta$-smoothed losses. Combining the results of Eq. \eqref{eq:rftl_regret_rftl_lemma},  \eqref{eq:cg_regret_rftl_lemma} and  \eqref{eq:osb_regret_rftl_lemma}, using the convexity of $f_t$, $\hat{f}_{t,\delta}(\mathbf{x}) -  \hat{f}_{t,\delta}(\mathbf{y}) \leq \nabla \hat{f}_{t,\delta}(\mathbf{x}) (\mathbf{x} - \mathbf{y})$, we obtain
    {
    \begin{align}
        & \sum_{t=1}^{T}  \left( \mathbb{E} \left[ \hat{f}_{t,\delta}(\mathbf{x}_{m(t)-1}) \right]  -  \hat{f}_{t,\delta}(\tilde{\mathbf{x}}^*) \right) \nonumber \\
        & \leq  \sum_{t=1}^{T} \mathbb{E} \left[ \hat{\nabla}_{t,\delta,m(t)-1}^{\top} (\mathbf{x}_{m(t)-1} - \tilde{\mathbf{x}}^*) \right] \nonumber \\
        & =  \sum_{t=1}^{T} \mathbb{E} \Big[ \hat{\nabla}_{t,\delta,m(t)-1}^{\top} (\mathbf{x}_{m(t)-1} - \mathbf{x}_{m(t)-1}^*)  \nonumber \\
        &~+ \hat{\nabla}_{t,\delta,m(t)-1}^{\top} (\mathbf{x}_{m(t)-1}^* - \mathbf{x}_{m(t)}^*) \nonumber\\
        &~+ \hat{\nabla}_{t,\delta,m(t)-1}^{\top} (\mathbf{x}_{m(t)}^* - \tilde{\mathbf{x}}^*) \Big] \nonumber \\
        & \leq  \left( G  \sqrt{\epsilon} +   \frac{\eta n M}{\delta} \left( G \sqrt{K} + \frac{n M}{\delta} \right) + 2 \eta K G^2 \right)T \nonumber \\
        &~+ \frac{4R^2}{\eta}. \label{eq:regret_delta_smooth_rftl_lemma}
    \end{align}}
    Using Lemma \ref{lemma:hazan_smooth} and the above equation, we have
    {\begin{align}
        & \sum_{t=1}^{T} \mathbb{E}  \left[ f_t(\mathbf{x}_{m(t)-1}) \right]  -  \sum_{t=1}^{T} f_t(\tilde{\mathbf{x}}^*) \nonumber \\
        & = \sum_{t=1}^{T} \mathbb{E} \left[ f_t(\mathbf{x}_{m(t)-1}) -  \hat{f}_{t,\delta}(\mathbf{x}_{m(t)-1}) \right] + \left( \hat{f}_{t,\delta} (\tilde{\mathbf{x}}^*) -  f_t(\tilde{\mathbf{x}}^*) \right) \nonumber \\
        & ~~ + \sum_{t=1}^{T} \left( \mathbb{E} \left[ \hat{f}_{t,\delta}(\mathbf{x}_{m(t)-1}) \right] -  \hat{f}_{t,\delta}(\tilde{\mathbf{x}}^*) \right) \leq \frac{4R^2}{\eta} \nonumber \\
        & + \left( 2 \delta G + G  \sqrt{\epsilon} +   \frac{\eta n M G \sqrt{K}}{\delta} + \frac{\eta n^2 M^2}{\delta^2} + 2 \eta K G^2 \right)T.
    \end{align}}
    Combining the last equation with Eq. \eqref{eq:bandit_feedback_regret_rftl_lemma},  \eqref{eq:squeeze_regret_rftl_lemma} and Eq. \eqref{eq:full_regret_rftl_lemma}, we obtain the required bound.
\end{proof}

The following lemma is used to upper-bound the number of iterations required by the conditional gradient method, Algorithm \ref{alg:Cg}, to terminate on each invocation.\begin{lemma}
    \label{lemma:cg_epsilon_error_L_iterations} Given a function $F(\mathbf{x})$, $2$-smooth and $2$-strongly convex, and $\mathbf{x}_{1} \in \mathcal{K}_{\delta}$ such that $F(\mathbf{x}_{1}) - F(\mathbf{x}^*) \leq \tilde{\epsilon}$, where $\mathbf{x}^* = \argmin\limits_{\mathbf{x} \in \mathcal{K}_{\delta}}F(\mathbf{x}) $, Algorithm \ref{alg:Cg} produces a point $\mathbf{x}_{L+1}\in\mK_{\delta}$ such that $F(\mathbf{x}_{L+1}) - F(\mathbf{x}^*) \leq \epsilon$ after at most $L = \max \bigg{\{} \frac{16R^2}{\epsilon^2} (h_1 - \epsilon) , 
    ~ \frac{2}{\epsilon} (h_1 - \epsilon) \bigg{\}}$ iterations.
\end{lemma}
\begin{proof}
    For any iteration $\tau$ of Algorithm \ref{alg:Cg}, define $h_{\tau} =  F(\mathbf{x}_{\tau}) - F(\mathbf{x}^*)$ and denote $\nabla_{\tau} = \nabla F(\mathbf{x}_{\tau})$. From the choice of $\v_{\tau}$ and the convexity of $F(\cdot)$, it follows that
    \begin{align}
        \nabla_{\tau}^{\top} (\mathbf{x}_{\tau} - \mathbf{v}_{\tau}) \geq & \nabla_{\tau}^{\top} (\mathbf{x}_{\tau} - \mathbf{x}^*) \nonumber \\
        \geq & F(\mathbf{x}_{\tau}) - F(\mathbf{x}^*) = h_\tau \label{eq:aux1_cg_lemma}
    \end{align}
    Now, we establish the convergence rate of Algorithm \ref{alg:Cg}. It holds that
    \begin{align}
        h_{\tau+1} = & F(\mathbf{x}_{\tau+1}) -  F(\mathbf{x}^*) \nonumber \\
        = & F(\mathbf{x}_{\tau} + \sigma_{\tau}(\mathbf{v}_{\tau} - \mathbf{x}_{\tau})) - F(\mathbf{x}^*) . \nonumber 
    \end{align}
    For our analysis we define the step-size $\hat{\sigma}_{\tau} = \min \Big{\{} \frac{ \nabla_{\tau}^{\top} (\mathbf{x}_{\tau} - \mathbf{v}_{\tau})}{8R^2}, 1 \Big{\}}$. 
    Since $\sigma_{\tau}$ is chosen via line-search, we have that
    \begin{align}
        h_{\tau+1} = &  F(\mathbf{x}_{\tau} + \sigma_{\tau}(\mathbf{v}_{\tau} - \mathbf{x}_{\tau})) - F(\mathbf{x}^*) \nonumber \\
        \leq & F(\mathbf{x}_{\tau} + \hat{\sigma}_{\tau}(\mathbf{v}_{\tau} - \mathbf{x}_{\tau})) - F(\mathbf{x}^*). \nonumber
    \end{align}
    Since $F(\mathbf{x})$ is $2$-smooth it holds that 
    \begin{align*}
    F(\mathbf{x}_{\tau} + \hat{\sigma}_{\tau}(\mathbf{v}_{\tau} - \mathbf{x}_{\tau})) &\leq F(\mathbf{x}_{\tau}) + \hat{\sigma}_{\tau} \nabla_{\tau}^{\top} (\mathbf{v}_{\tau} - \mathbf{x}_{\tau}) \\
    &+ \hat{\sigma}_{\tau}^2  \Vert \mathbf{v}_{\tau} - \mathbf{x}_{\tau} \Vert ^2, 
    \end{align*}
    and we obtain
    \begin{align}
        h_{\tau+1} \leq & h_{\tau} + \hat{\sigma}_{\tau}^2 (2R)^2 - \hat{\sigma}_{\tau} \nabla_{\tau}^{\top} (\mathbf{x}_{\tau} - \mathbf{v}_{\tau}).   \nonumber
    \end{align}
    We now consider several cases.
    \\ \emph{Case 1}: If $\nabla_{\tau}^{\top} (\mathbf{x}_{\tau} - \mathbf{v}_{\tau}) \leq \epsilon$ for some $\tau < L$, the algorithm will stop after less than $L$ iterations. Moreover, from Eq. \eqref{eq:aux1_cg_lemma} we have $h_{\tau} \leq \epsilon$.
    \\ \emph{Case 2}: Else, $\nabla_{\tau}^{\top} (\mathbf{x}_{\tau} - \mathbf{v}_{\tau}) \geq \epsilon$ for all $\tau < L$. We have 2 cases: \\
    \emph{Case 2.1}: If $\nabla_{\tau}^{\top} (\mathbf{x}_{\tau} - \mathbf{v}_{\tau}) \geq 8R^2$ then $\hat{\sigma}_{\tau} = 1$ and we have
    \begin{align}
        h_{\tau+1} \leq & h_{\tau} + \hat{\sigma}_{\tau}^2 (2R)^2 - \hat{\sigma}_{\tau} \nabla_{\tau}^{\top} (\mathbf{x}_{\tau} - \mathbf{v}_{\tau})   \nonumber \\
        \leq & h_{\tau} - \frac{ \nabla_{\tau}^{\top} (\mathbf{x}_{\tau} - \mathbf{v}_{\tau})}{2} .\nonumber
    \end{align}
        \\ \emph{Case 2.2}: Else, $ \nabla_{\tau}^{\top} (\mathbf{x}_{\tau} - \mathbf{v}_{\tau}) \leq 2D^2$, and then $\hat{\sigma}_{\tau} = \frac{ \nabla_{\tau}^{\top} (\mathbf{x}_{\tau} - \mathbf{v}_{\tau})}{8R^2}$, and we have
    \begin{align}
        h_{\tau+1} \leq & h_{\tau} + \hat{\sigma}_{\tau}^2 (2R)^2 - \hat{\sigma}_{\tau}  \nabla_{\tau}^{\top} (\mathbf{x}_{\tau} - \mathbf{v}_{\tau})   \nonumber \\
        \leq & h_{\tau} - \left( \frac{\nabla_{\tau}^{\top} (\mathbf{x}_{\tau} - \mathbf{v}_{\tau})}{4R} \right)^2. \nonumber
    \end{align}
    From both cases, we have
    \begin{align}
        h_{\tau+1} & \leq h_{\tau} - \min \bigg{\{} \left( \frac{ \nabla_{\tau}^{\top} (\mathbf{x}_{\tau} - \mathbf{v}_{\tau})}{4R} \right)^2, \frac{ \nabla_{\tau}^{\top} (\mathbf{x}_{\tau} - \mathbf{v}_{\tau})}{2} \bigg{\}} \nonumber \\
        \leq & h_{1} - \tau \min_{i = 1, \dots, \tau} \bigg{\{} \left( \frac{ \nabla_{\tau}^{\top} (\mathbf{x}_{i} - \mathbf{v}_{i})}{4R} \right)^2, \frac{ \nabla_{\tau}^{\top} (\mathbf{x}_{i} - \mathbf{v}_{i})}{2} \bigg{\}} \nonumber \\
        & \leq  h_{1} - \tau \min \bigg{\{} \left( \frac{ \epsilon}{4R} \right)^2, \frac{ \epsilon}{2} \bigg{\}}.
    \end{align}
    Thus, for all cases, after a maximum of $L$ iterations, when
    \begin{align}
        L = \max \bigg{\{} \frac{16R^2}{\epsilon^2} (h_1 - \epsilon) , 
        ~ \frac{2}{\epsilon} (h_1 - \epsilon) \bigg{\}}, \nonumber
    \end{align}
    we obtain $ h_{L+1} \leq \epsilon$.
\end{proof}

We can now finally prove our main theorem, Theorem \ref{thm:main}.

\begin{proof}[Proof of Theorem \ref{thm:main}] 
We first upper bound the expected overall number of calls to the linear optimization oracle throughout the run of the algorithm, and then we upper-bound the expected regret.

Let $\mathbf{z}_{m,\tau}$ be the iterate of Algorithm \ref{alg:Cg} after completing $\tau-1$ iterations of the do-while loop, when invoked on iteration (block) $m$ of Algorithm \ref{alg:VRBCO}. Also, for all $m,\tau$, define $h_{m,\tau} := \hat{F}_{m}(\mathbf{z}_{m,\tau}) -  \hat{F}_{m}(\mathbf{x}_{m}^*)$. Recall that for any iteration $m$ of Algorithm \ref{alg:VRBCO}, we have $\mathbf{z}_{m,1} = \mathbf{x}_{m-1}$.

Using the triangle inequality and the fact $\hat{F}_{m}(\mathbf{x}_{m+1}^*) \geq \hat{F}_{m}(\mathbf{x}_{m}^*)$, we have
\begin{align}
    \mathbb{E}[h_{m+1,1}] & =  \mathbb{E}[\hat{F}_{m+1}(\mathbf{z}_{m+1,1}) -  \hat{F}_{m+1}(\mathbf{x}_{m+1}^*)] \nonumber \\
    \leq  \mathbb{E}  [ & \hat{F}_{m}(\mathbf{x}_{m}) -  \hat{F}_{m}(\mathbf{x}_{m}^*) + \eta   \Vert \hat{\mathbf{g}}_{m} \Vert ~  \Vert  \mathbf{x}_{m} - \mathbf{x}_{m+1}^* \Vert ]. \nonumber
\end{align}
Since $h_{m,L_{m}} = \hat{F}_{m}(\mathbf{x}_{m}) -  \hat{F}_{m}(\mathbf{x}_{m}^*) \leq \epsilon$, using the triangle inequality, we have
\begin{align}
    \mathbb{E}[h_{m+1,1}&]  \leq  \epsilon +  \eta \mathbb{E}[  \Vert \hat{\mathbf{g}}_{m} \Vert \Vert \mathbf{x}_{m} - \mathbf{x}_{m}^* + \mathbf{x}_{m}^* - \mathbf{x}_{m+1}^* \Vert ] \nonumber \\
    \leq & \epsilon + \eta \mathbb{E}[    \Vert \hat{\mathbf{g}}_{m} \Vert  \left(  \Vert \mathbf{x}_{m} - \mathbf{x}_{m}^* \Vert  +  \Vert \mathbf{x}_{m}^* - \mathbf{x}_{m+1}^* \Vert  \right) ]. \nonumber
\end{align}
Since $\hat{F}_{m}(\mathbf{x})$ is $2$-strongly convex and $h_{m,L_{m}} \leq \epsilon$, using Eq. \eqref{eq:strong_convexity}, we have that $ \Vert \mathbf{x}_{m} - \mathbf{x}_{m}^* \Vert  \leq \sqrt{\epsilon}$. Also, from Eq. \eqref{eq:absolut_dist}, we have $ \Vert \mathbf{x}_{m}^* - \mathbf{x}_{m+1}^* \Vert  \leq \eta  \Vert \hat{\mathbf{g}}_{m} \Vert $. Thus, we have
\begin{align}
    \mathbb{E}[h_{m+1,1}] \leq & \epsilon  +  \eta \sqrt{\epsilon}~ \mathbb{E}[ \Vert \hat{\mathbf{g}}_{m} \Vert ]  + \eta^2 \mathbb{E}[ \Vert \hat{\mathbf{g}}_{m} \Vert ^2]. \nonumber
\end{align}
Using Lemma \ref{lemma:expectation_gradient} and the fact that for all $a,b \in \reals^+$ it holds that that $\sqrt{a+b} \leq \sqrt{a} + \sqrt{b}$, we have
\begin{align}
    \mathbb{E}[h_{m+1,1}] \leq & \epsilon  + \eta  \sqrt{\epsilon} ~ \left( \sqrt{K} \left(\frac{n M}{\delta}\right) + K G \right) \nonumber \\
    & + \eta^2 \left(K \left(\frac{n M}{\delta}\right)^2 + K^2 G^2 \right). \label{eq:upper_bound_h_1_cg_lemma}
\end{align}
Using Lemma \ref{lemma:cg_epsilon_error_L_iterations} with $\tilde{\epsilon} = \epsilon  + \eta  \sqrt{\epsilon} ~ \left( \sqrt{K} \left(\frac{n M}{\delta}\right) + K G \right) + \eta^2 \left(K \left(\frac{n M}{\delta}\right)^2 + K^2 G^2 \right)$ for $m = 1, \dots, \frac{T}{K}$, we have that on each iteration (block) $m$, the number of calls to the linear optimization oracle is $L_m \leq \max \bigg{\{} \frac{16R^2}{\epsilon^2} (h_{m,1} - \epsilon) ,~ \frac{2}{\epsilon} (h_{m,1} - \epsilon) \bigg{\}}$. Plugging-in $\epsilon$, we have $ L_{m}\leq \frac{16R^2}{\epsilon^2} (h_{m,1} - \epsilon)$.
Following Eq. \eqref{eq:upper_bound_h_1_cg_lemma} we have 
\begin{align}
    \mathbb{E}[L_{m}] \leq & \frac{16R^2}{\epsilon^2} (\mathbb{E}[h_{m,1}] - \epsilon) \nonumber \\
    \leq & \frac{16R^2}{\epsilon^2} \eta  \sqrt{\epsilon} ~ \left( \sqrt{K} \left(\frac{n M}{\delta}\right) + K G \right)  \nonumber \\
    & + \frac{16R^2}{\epsilon^2} \eta^2 \left( K \left(\frac{n M}{\delta}\right)^2 + K^2 G^2 \right) \nonumber \\
    \underset{(a)}{=} & ~ \left( \frac{3}{4} + \frac{G c}{2 n M} + \frac{G^2 c^2 }{4 n^2 M^2} \right) T^{\frac{1}{2}}.\nonumber
\end{align}
    Equality (a) is due to plugging-in $\eta, \delta, \epsilon, K$.
    Thus, overall on all blocks, we obtain 
    \begin{align}
        \mathbb{E} \left[ \sum_{m=1}^{\frac{T}{K}} L_m \right] \leq & ~\left( \frac{3}{4} + \frac{G c}{2 n M} + \frac{G^2 c^2 }{4 n^2 M^2} \right) T. \nonumber
    \end{align}
    We now turn to upper-bound the expected regret of the algorithm. Using Lemma \ref{lemma:RFTL} we have that
\begin{align}
    & \mathbb{E}[\mathcal{R}_{T}] \leq \left( 3 \delta G + \delta{}RG/r + G  \sqrt{\epsilon} +   \frac{\eta G \sqrt{K} n M}{\delta} \right. \nonumber \\ & ~~~~~~~~~~~~~~~~~ \left. + \eta \left( \frac{n M}{\delta} \right)^2  + 2 \eta K G^2 \right)T + \frac{4R^2}{\eta}  \nonumber \\
    & \underset{(a)}{=} \left( 3  c  G + \frac{c R G }{r} + 6 G R + 4 \frac{c G^2 R}{nM} + 4\frac{RnM}{c } \right) T^{\frac{3}{4}}. \nonumber
\end{align}
Equality (a) is due to plugging-in $\eta, \delta, \epsilon, K$.
\end{proof}


\bibliographystyle{plain}
\bibliography{bib}

\end{document}